\documentclass[sigconf,nonacm]{acmart}
\setcopyright{none}
\renewcommand\footnotetextcopyrightpermission[1]{}

\usepackage{booktabs}
\usepackage{array}
\usepackage{amsmath}
\graphicspath{{figures/}}
\newtheorem{thm}{Theorem}
\newtheorem{prop}{Proposition}
\newtheorem{cor}{Corollary}

\title{Smart Alone, Blind Together: Why Accuracy Fails to Certify Theory of Mind in Multi-Agent Teams}

\begin{document}
\author{Will Beaumaster}
\affiliation{\institution{University of Minnesota}\city{Minneapolis}\state{MN}\country{USA}}
\email{beaum045@umn.edu}
\author{Paul Schrater}
\affiliation{\institution{University of Minnesota}\city{Minneapolis}\state{MN}\country{USA}}
\email{schrater@umn.edu}

\begin{abstract}
Coordinators that act on what they infer about teammates are usually certified by the
accuracy of those inferences. We argue that accuracy is the wrong certificate, for two
reasons. First, the accuracy needed to act depends on what a wrong action costs, so the
bar belongs to the deployment, not the module. Second, a module's errors are rarely
independent, and dependence has a sign: positively associated errors help where their
damage overlaps and hurt where it compounds, so an accuracy threshold fails in
a predictable direction, and once wrong beliefs interact and can do harm, no rule based
on accuracy is right for every module. We test these predictions in task assignment, Level-Based
Foraging and Hanabi, with injected errors, trained and Bayesian modules and five language
models, and in human judgments. Accuracy fails in both directions,
approving harmful language-model deployments in Hanabi, where shared errors compound, and
rejecting beneficial ones in task assignment, where costly errors are rare. In our tasks it
also buys far less redundancy than independence predicts: from GPT-5.4-mini to Claude
Opus~5.5, error on teammates the models have evidence about falls from 32\% to 11\%, but the
chance that all three observers misread one falls only from 16\% to 9\% (independence: 3.4\%
to 0.15\%). We propose certifying the deployment, with a Blind-Spot Gate that prices error
costs, pairwise interactions and measured dependence without deploying the module. On held-out data from the cells it
was developed on, it makes no significant false approval in 20 splits (accuracy: 3.2 per split), and in
preregistered tests on modules it was not developed on, it recovers what accuracy rejects
(seven classifiers: regret 0.004 against 0.479) and catches what accuracy approves (a
Bayesian module: significant false approvals in 17 of 20 splits for accuracy, 1 for
the gate).
\end{abstract}


\maketitle

\section{Introduction}
A coordinator that models its teammates acts on what it infers: it assigns a
worker to the task it believes the worker wants, gives a Hanabi hint in the
convention it believes the receiver uses, or lets a language model's reading of a
partner decide what happens next. A wrong inference does not merely fail to
help. It moves the system somewhere a passive, belief-free policy would not have
gone, so below some accuracy, acting on theory of mind (ToM) is worse than not
acting at all. ToM modules, language models among them, are evaluated mostly by the
accuracy of their inferences \cite{riemer2025position}, and deployment decisions
follow: a module is trusted if its accuracy clears a threshold.

We argue that accuracy is the wrong certificate, for two reasons. The first is cost.
For one belief, the break-even accuracy is a cost-sensitive threshold \cite{elkan2001}:
it is set by what a wrong action costs in the deployment, not by how hard the inference
is. In three domains the threshold slides with that cost, and in two it stays put for
errors that cannot reach the costly channel, the events through which a wrong action
does harm.

The second reason is dependence, the main subject of the paper. A module infers many
beliefs, and its errors are rarely independent: one ambiguous observation misleads it
about several teammates, and a belief inferred in public is wrong for every observer at
once. Dependence has a sign. Where two errors do less harm together than apart (their
damage overlaps), positively associated errors raise the expected benefit of acting;
where they do more (damage compounds), they lower it (Theorem~\ref{thm:corr};
Figure~\ref{fig:schematic}). An accuracy threshold calibrated on independent errors
therefore fails in a predictable direction: too strict when errors overlap in what they
break, too lenient when a shared misreading removes the redundancy that lets one
teammate cover another (Corollary~\ref{cor:threshold}). And once wrong beliefs interact
and acting on them can hurt, no rule based on accuracy is right for every module; a
certificate must see how errors co-occur (Theorem~\ref{thm:order}).

Real modules make these errors. Trained classifiers in task assignment rarely make
the costly kind, and accuracy rejects deployments that help in 33 of 35 combinations of
module and harm. A Bayesian module in Hanabi shares its mistakes across observers, and
accuracy approves deployments that hurt. Language models do both, and show a further
regularity: in our tasks, accuracy buys far less redundancy than independence
predicts, because accurate models' remaining errors are largely shared across copies
(Proposition~\ref{prop:neff}): three observers of a teammate are worth
\(n_{\mathrm{eff}}\le3\) independent ones, and \(n_{\mathrm{eff}}\) approaches 1 as a model's
remaining errors become blind spots that every copy shares. The models are accurate, but
wrong together. People show a
milder version of the same pattern. We therefore propose certifying the deployment: the
Blind-Spot Gate prices what errors cost, how they interact and how often they co-occur,
from replays with injected errors and the module's error statistics, without deploying the
module.

\paragraph{Contributions.} (1) A cost-based account of when acting on ToM pays, with
preregistered tests in three domains (\S\ref{sec:model}, \S\ref{sec:cost}). (2) A theory
of certification under correlated errors: the sign condition of Theorem~\ref{thm:corr}
and Corollary~\ref{cor:threshold}, which tells which way accuracy fails, and
Theorem~\ref{thm:order}, which says that no accuracy-based rule is right for every
module once wrong beliefs interact and acting on them can hurt. (3) Evidence from classifiers, a Bayesian module and five language models (four Claude,
one GPT) that accuracy fails in both directions and buys far less redundancy than
independence predicts, the latter also in human readers (\S\ref{sec:real}--\S\ref{sec:human}). (4) The Blind-Spot Gate: as safe as an A/B test with a 95\% bound on
held-out data of the modules it was developed on, better than measuring the deployment
directly when games of the module are scarce (Hanabi), and lower in regret than accuracy on
classifiers and a Bayesian module it was not developed on (preregistered; \S\ref{sec:gate}). The individual tools are known: cost-sensitive
thresholds, association and the supermodular order \cite{esary1967,christofides2004},
M\"obius expansions. The contribution is the certification problem, the sign condition,
and the gate. Table~\ref{tab:map} lists what each claim rests on.

\begin{table}[t]
\centering
\footnotesize
\caption{What each claim rests on. Held: a prediction fixed before the test came out as
predicted; descriptive: a measurement reported without a prediction; exploratory: an
analysis chosen after seeing the data. IDs refer to \texttt{code\_data/SCORECARD.md} in the
supplement, which scores all 73 preregistered predictions of the tests cited here: 48 held,
3 held in part, 21 failed and 1 was unscored.}
\label{tab:map}
\setlength{\tabcolsep}{3pt}
\begin{tabular}{@{}>{\raggedright\arraybackslash}p{0.60\columnwidth}>{\raggedright\arraybackslash}p{0.36\columnwidth}@{}}
\toprule
claim (evidence) & status \\
\midrule
The break-even accuracy rises with the cost of a wrong action (task assignment, LBF, Hanabi) & held (LBF1--2, HA2) \\
Errors that cannot reach the costly channel leave it flat (task assignment, LBF) & held (P1--3; LBF4's under-estimate half) \\
Where damage compounds, shared errors cost benefit (Hanabi: injected errors, Bayesian module, Claude clones) & held (HC2, HD2, CB5, EF4) \\
\ldots\ and accuracy approves harmful deployments (Claude teams, Bayesian module, GPT-5.4-mini) & held (M7, CB7); module, GPT exploratory \\
With one life, where Hanabi is also a series unit, shared errors help (injected errors) & not held (HC1, G2): the two effects cancel \\
Where costly errors are rare, accuracy rejects good deployments (classifiers; Haiku~4.5, Sonnet~5.5) & held (E4); classifiers descriptive \\
Copies of one model fail together far beyond independence (Claude models, relabelled; GPT-5.4-mini) & held (M1, CB4); GPT exploratory \\
Accuracy buys far less redundancy than independence predicts (five models; three human datasets) & held (CB3, EF1); GPT and human data exploratory \\
Pricing costs, interactions and dependence certifies better (language models; out of sample, classifiers and a Bayesian module) & exploratory (57 language-model cells); held (OOS1--2, OOSH2--3); OOSH1 not held (one split) \\
Deferring on unread teammates never loses (Hanabi, 30 fresh-seed cells) & held (G1) \\
\bottomrule
\end{tabular}
\end{table}

\section{Related work}
\paragraph{Machine theory of mind.} Inferring others' hidden mental states from
behaviour and acting on the inferences is the setting of machine ToM
\cite{rabinowitz2018machine,baker2017rational}, now often with language models in
multi-agent games \cite{cross2025hypothetical}. That work mostly asks whether
inferred beliefs are accurate or whether a ToM module helps. The closest statement
of our thesis is that literal ToM accuracy and functional value come apart for
language models \cite{riemer2025position}; we show that accuracy misorders modules
by functional value in both directions, and predict which. Mismatched ToM depth
hurts LLM coordination \cite{mu2026adaptive}, and homogeneous LLM teams fall into
a symmetry trap that theory of mind cannot break, because every agent forms the same
prediction of the others
\cite{yuan2026scene}, an identical-agent case for which Theorem~\ref{thm:corr}
gives the sign condition. Public-belief methods in Hanabi share one inference among all
players by design \cite{foerster2019bayesian}, the case of our public teams.

\paragraph{When modelling others pays.} The value of ToM depends on the setting:
higher-order ToM helps most in unpredictable negotiations
\cite{deweerd2022unpredictable}, cooperation can favour less sophisticated ToM
\cite{devaine2014evolution}, and an inaccurate opponent model can be worse than
none \cite{goodman2020opponent}. In type-based ad hoc teamwork, acting on hypothesised
teammate types can succeed even when the hypotheses are wrong \cite{albrecht2016belief}.
Safe exploitation bounds the loss from acting on a
prediction and conjectures that predictor faults matter less in cooperative games
\cite{damer2017safely}, which our thresholds contradict when a wrong action is costly.
Whether to model other agents
at all, and judging models by the decisions they produce, not their
accuracy, are open problems \cite{albrecht2018modelling}; the
threshold answers the first quantitatively, and Theorem~\ref{thm:order} shows why
the second is needed. In Hanabi, ad hoc agents keep a Bayesian belief over which known
strategy a partner plays \cite{canaan2023adhoc}, as our module does over conventions.

\paragraph{Correlated errors and monoculture.} Language models err together:
models that are both wrong often give the same wrong answer, more so for more
accurate models \cite{kim2025correlated,goel2025great}, and aggregating samples
amplifies shared misconceptions
\cite{denisovblanch2026consensus}. How much agreement exceeds chance depends on the null
model, because items differ in difficulty \cite{jo2026subjectivity}.
Algorithmic monoculture lowers collective decision quality even when the shared
algorithm is individually better \cite{kleinberg2021algorithmic}. Ensemble learning relates
classifier diversity to ensemble accuracy \cite{kuncheva2003measures}, and with correlated
votes more voters add less than the Condorcet jury theorem promises \cite{ladha1992condorcet}.
Independently developed software versions fail together because inputs differ in
difficulty \cite{eckhardt1985theoretical}, and forced diversity changes how often
\cite{littlewood1989conceptual}; Proposition~\ref{prop:neff} carries these coincident-failure
models to observers of a teammate.
These works measure agreement or its cost in selection and voting; we ask what agreement does to a team
that acts on its beliefs, find that the sign depends on whether damage overlaps or
compounds, and turn the answer into a deployment test, building on reliability theory:
association \cite{esary1967} and the supermodular order \cite{christofides2004}.

\section{The threshold and what moves it}
\label{sec:model}
\paragraph{Setup.} A coordinator acts for a team whose \(W\) members have hidden
types (rules, levels, conventions), holding a belief about each. We measure
performance on a normalized scale, \(B=(\text{score}-\text{do-nothing})/
(\text{informed}-\text{do-nothing})\), where do-nothing is the passive, belief-free
policy (\(B=0\)) and informed is the same coordinator with true beliefs (\(B=1\)), assumed to score higher.
If acting on a correct belief contributes \(1/W\), acting on a confidently wrong
one contributes \(\ell(h)/W\), with \(h\) the cost of a wrong action, and
contributions add, then
\begin{equation}
B(a,h)=a+(1-a)\,\ell(h),\qquad a^{*}(h)=\frac{|\ell(h)|}{1+|\ell(h)|}\ \text{ if }\ell(h)<0 .
\label{eq:meanfield}
\end{equation}
For one belief this is the cost-sensitive threshold of \cite{elkan2001}. Its
knee, the accuracy \(a^{*}\) below which acting hurts, exists only if acting on a wrong belief is worse than not acting (\(\ell<0\)), which requires a passive equilibrium that a wrong
action can damage, and the threshold rises with \(|\ell|\). Two results need
no additive form.

\begin{prop}[Cost and channels]
\label{prop:mono}
Let \(a^{*}(h)=\sup\{a:B(a,h)<0\}\). (a) If raising harm never turns a harmful
coordinator into a beneficial one (\(B(a,h)<0\Rightarrow B(a,h')<0\) for \(h'>h\)),
then \(a^{*}(h)\) is non-decreasing in \(h\). This holds automatically when harm is
a pure cost: trajectories do not depend on \(h\) and \(R_h=R_0-h\,c\), with
\(c\ge0\) charged only on a set \(E\) of events that neither do-nothing nor the
informed policy triggers, so that \(B(a,h)=B(a,0)-h\,\kappa(a)\), \(\kappa\ge0\).
(b) If the coordinator's policy does not depend on \(h\), harm acts only through \(E\)
(trajectories and rewards agree for all \(h\) on every history without an event in
\(E\)), and neither reference policy triggers \(E\), a class of belief errors that never triggers an event in \(E\) has a
threshold that does not depend on \(h\).
\end{prop}
\begin{proof}
(a) The hypothesis makes \(\{a:B(a,h)<0\}\) grow with \(h\), so its supremum cannot
fall; under a pure cost the normalizer is \(h\)-invariant and \(\kappa(a)\) is the
expected charge over it. (b) Such errors produce the same trajectories and rewards
at every \(h\), and neither reference policy depends on \(h\).
\end{proof}
LBF's penalty is a pure cost. In task assignment (ejection) and Hanabi (lives)
harm feeds back into the dynamics, and (a)'s hypothesis is an empirical fact: in
our runs the curves for different harms cross only where all of them are negative,
so a coordinator harmful at one harm level is harmful at every higher one.

\section{Correlated errors}
\label{sec:corr}
Fix the wrong value each belief would take, let \(\varepsilon\in\{0,1\}^P\) record
which of \(P\) beliefs are wrong, and write \(\bar B(\varepsilon)\) for the
expected benefit given \(\varepsilon\). Call \(\bar B\) \emph{supermodular} when an
extra wrong belief costs less if others are already wrong (damage overlaps) and
\emph{submodular} when it costs more (damage compounds). A random vector is
\emph{positively associated} when every two non-decreasing functions of it have
non-negative covariance \cite{esary1967}, and it dominates another vector in the
\emph{supermodular order} when every supermodular function has at least as large an
expectation under it \cite{christofides2004}.

\begin{thm}[Correlated errors]
\label{thm:corr}
Let \(\varepsilon\) be positively associated \cite{esary1967} and let
\(\varepsilon^{\perp}\) have independent components with the same marginals. If
\(\bar B\) is supermodular, \(\mathbb{E}\,\bar B(\varepsilon)\ge\mathbb{E}\,\bar
B(\varepsilon^{\perp})\); if \(\bar B\) is submodular, the inequality reverses.
\end{thm}
\begin{proof}
By \cite{christofides2004}, an associated vector dominates its independent version in
the supermodular order; apply the order to \(\bar B\), and in the submodular case to
\(-\bar B\), which is then supermodular.
\end{proof}

Errors driven by shared latent causes, such as one ambiguous observation or one
public belief per teammate held by every observer, are associated when each error
is non-decreasing in independent causes \cite{esary1967}. Which case applies can
often be read off the task.

\begin{prop}[Overlap and compounding from task structure]
\label{prop:units}
Let the benefit be a sum over units (cards, tasks, teammates), unit \(u\) losing value
with the number of wrong beliefs among those it depends on, \(C_u\):
\(\bar B(\varepsilon)=b-\sum_u w_u\,\phi_u\bigl(\sum_{p\in C_u}\varepsilon_p\bigr)\),
\(w_u\ge0\), \(\phi_u\) non-decreasing. If every \(\phi_u\) is concave, \(\bar B\) is
supermodular; if every \(\phi_u\) is convex, \(\bar B\) is submodular. In particular a
unit lost when \emph{any} of its beliefs is wrong (\(\phi(k)=\min(k,1)\)) makes damage
overlap, and a unit lost only when \emph{all} are wrong (\(\phi(k)=\mathbf 1[k=|C_u|]\),
redundancy) makes it compound: the series and parallel systems of reliability
theory \cite{esary1967}.
\end{prop}
\begin{proof}
In two coordinates \(p\ne q\) the mixed difference of
\(\phi_u(\sum_{C_u}\varepsilon)\) is \(\phi_u(k{+}2)-2\phi_u(k{+}1)+\phi_u(k)\), with \(k\)
the other wrong beliefs in \(C_u\), if \(p,q\in C_u\), and \(0\) otherwise: \(\le0\) for
concave and \(\ge0\) for convex \(\phi_u\). Supermodularity on \(\{0,1\}^P\) is
nonnegativity of all mixed differences, which nonnegative sums preserve and the minus
sign reverses.
\end{proof}

Two beliefs each wrong with probability 0.2 show the size of the effect
(Figure~\ref{fig:schematic}). A unit lost
only when both are wrong (redundancy) is lost 4\% of the time if the errors are
independent and 20\% if they are fully shared; a unit lost when either is wrong, 36\%
and 20\%: the same accuracy and dependence move the benefit in opposite directions.

\begin{figure}[t]
\centering
\includegraphics[width=\columnwidth]{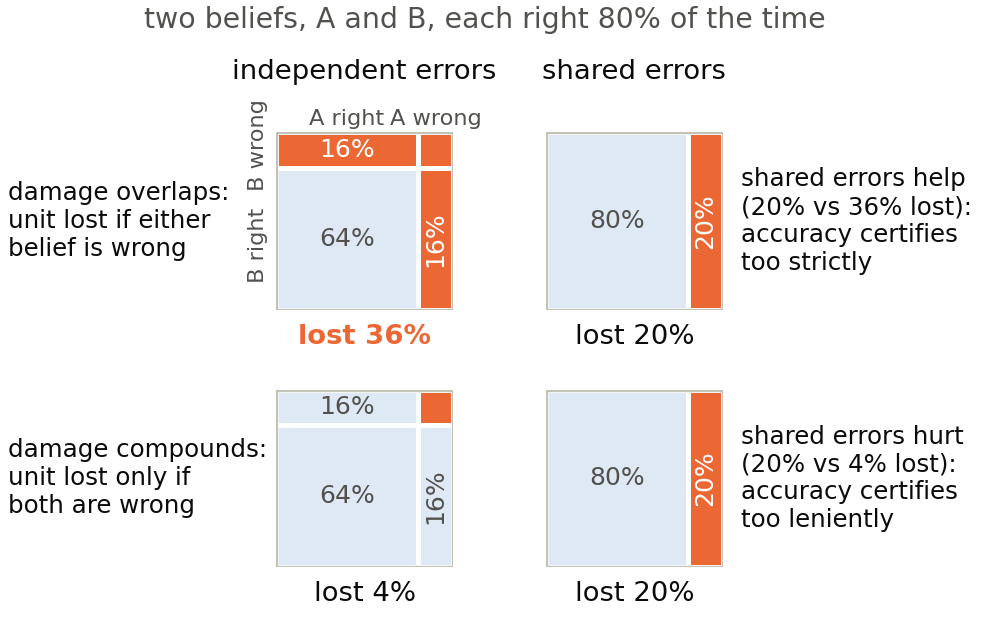}
\caption{Why dependence helps in one task and hurts in another. Two beliefs, each right
80\% of the time; each square shows the joint outcomes, orange where the unit is lost.
Shared errors move probability out of the one-wrong cells, half to both-wrong and half to
neither-wrong. A unit lost when either
belief is wrong (damage overlaps) is then lost less often, 20\% against 36\%; a unit lost
only when both are wrong (damage compounds) more often, 20\% against 4\%. Accuracy is
80\% in all four squares and sees no difference.}
\label{fig:schematic}
\end{figure}

Hanabi, read this way, has both kinds of unit: a teammate's critical cards are
mostly lost when no observer reads the teammate correctly (redundant), and with one
life a misreading can end the game outright (a series unit over all beliefs). If the
redundant units dominate, as they should with two or more lives, the proposition predicts
compounding, which \S\ref{sec:real} finds. With one life the two kinds pull in opposite
directions; we had preregistered a net gain from sharing there, but the two effects
cancel (Table~\ref{tab:map}). Coherence (wrong beliefs sharing a wrong value) is outside
the theorem; its gain in task assignment belongs to blocks of look-alike
workers, not pairs (supplement).

\begin{cor}[Certification under associated errors]
\label{cor:threshold}
Fix the per-belief accuracy \(a\); let \(B^{\perp}(a)\) and \(B^{\mathrm{as}}(a)\)
be the expected benefit under independent errors and under a positively associated
error law with those marginals (one for each \(a\)), with thresholds \(a^{*}_{\perp}\) and \(a^{*}_{\mathrm{as}}\).
If \(\bar B\) is supermodular, \(a^{*}_{\mathrm{as}}\le a^{*}_{\perp}\); if
submodular, \(a^{*}_{\mathrm{as}}\ge a^{*}_{\perp}\). A rule that deploys whenever
\(a>a^{*}_{\perp}\) therefore rejects deployments that do not hurt at
\(a\in(a^{*}_{\mathrm{as}},a^{*}_{\perp}]\) when damage overlaps, and, when
damage compounds and \(B^{\mathrm{as}}\) is increasing in \(a\), approves harmful
ones at \(a\in(a^{*}_{\perp},a^{*}_{\mathrm{as}})\).
\end{cor}
\begin{proof}
By Theorem~\ref{thm:corr}, \(B^{\mathrm{as}}\ge B^{\perp}\) pointwise in the
supermodular case, so \(\{a:B^{\mathrm{as}}(a)<0\}\subseteq\{a:B^{\perp}(a)<0\}\);
the submodular case reverses the inclusion. Above \(a^{*}_{\mathrm{as}}\),
\(B^{\mathrm{as}}\ge0\) by the definition of a supremum; below it,
\(B^{\mathrm{as}}<0\) when \(B^{\mathrm{as}}\) is increasing.
\end{proof}

The corollary gives the direction of the failure under association; where acting on
wrong beliefs can hurt, no rule based on accuracy avoids the failure itself. Write
\(\bar B(\varepsilon)=\sum_S\beta(S)\prod_{p\in S}\varepsilon_p\) (\(\beta(S)\): the interaction of
the beliefs in \(S\)); a module, any distribution of \(\varepsilon\), has expected benefit
\(\sum_S\beta(S)\,m(S)\), with \(m(S)\) the probability that every belief in \(S\) is wrong, and
its \(m(S)\) with \(|S|\le k\) are its \emph{order-\(k\) error rates} (order 1 is accuracy).

\begin{thm}[What a certificate must see]
\label{thm:order}
Let \(\bar B(\mathbf 1)<0<\bar B(\mathbf 0)\) (all beliefs wrong hurts, all right helps). A rule that knows \(\bar B\) and a module's order-\(k\) error
rates can act exactly when the module's expected benefit is positive, for every module,
iff \(\beta(S)=0\) whenever \(|S|>k\). Otherwise there are two modules with the same
order-\(k\) error rates, every belief of each right with the same probability, whose
expected benefits have opposite signs.
\end{thm}
\begin{proof}[Proof sketch]
Sufficiency is the formula above. For necessity take \(|T|>k\) with \(\beta(T)\ne0\) and
independent errors at accuracy \(a\in(0,1)\); adding a small \(\eta(-1)^{|T|-|x|}\) to the
probability of each pattern \(x\subseteq T\) of wrong beliefs moves \(m(T)\) alone, since for every \(S\ne T\) the
contributions to \(m(S)\) cancel (M\"obius inversion), so modules with
these order-\(k\) rates have benefits filling some \([L(a),U(a)]\), \(L<U\). As optimal values of
linear programs whose right-hand sides, the order-\(k\) rates, lie in the relative
interior of the set of rate vectors some module attains, for every \(a\in(0,1)\), \(L\) and \(U\) are continuous in \(a\); a union bound gives \(U<0\) near \(a=0\) and \(L>0\) near \(a=1\), so
\(L(a)<0<U(a)\) just above the largest \(a\) with \(U(a)\le0\) (full proof in Appendix~\ref{app:thm2}).
\end{proof}

Accuracy (\(k=1\)) thus certifies every module only if no wrong beliefs interact.
Figure~\ref{fig:domains}c shows two such modules: with five lives, at accuracies between the two
thresholds (0.62 and 0.73), independent errors benefit and errors shared by every
observer harm.

What is specific to theory of mind is the source of the dependence: an inference made
once and shared publicly makes every observer's error about a target the same event,
the most dependent errors possible at that accuracy, and copies of one model approach
this as their remaining errors become blind spots (Proposition~\ref{prop:neff}).

\begin{prop}[Redundancy among observers]
\label{prop:neff}
Let \(n\ge2\) observers each infer a target \(t\), with errors independent given
the target. (a) If they share a target-level error probability \(p_t\) (samples of
one model given the same evidence, or readers drawn at random given the item) with \(0<\mathbb{E}\,p_t<1\), then
\(\Pr[\text{all }n\text{ wrong}]=\mathbb{E}[p_t^{\,n}]\ge(\mathbb{E}\,p_t)^n\), and
\(n_{\mathrm{eff}}=\ln\Pr[\text{all }n\text{ wrong}]/\ln\mathbb{E}\,p_t\) lies in
\([1,n]\), with \(n_{\mathrm{eff}}=n\) iff \(p_t\) is constant and
\(n_{\mathrm{eff}}=1\) iff \(p_t\in\{0,1\}\) almost surely. (b) If observer \(i\) has
its own \(p_{t,i}\) (different models),
\(\Pr[\text{all wrong}]=\mathbb{E}\prod_ip_{t,i}\le\prod_i(\mathbb{E}\,p_{t,i}^{\,n})^{1/n}\):
a mixed team fails together at most as often as the geometric mean of its members'
clone teams, with equality iff their error patterns \(p_{t,i}^{\,n}\) are proportional.
\end{prop}
\begin{proof}
(a) Total probability gives the equality and Jensen's inequality the lower bound,
hence \(n_{\mathrm{eff}}\le n\); \(p^n\le p\) gives \(n_{\mathrm{eff}}\ge1\), with
equality iff \(p\in\{0,1\}\). (b) is H\"older's inequality with \(n\) exponents equal
to \(n\).
\end{proof}

\(n_{\mathrm{eff}}\) is the number of independent observers of the same accuracy
that would fail as rarely. Removing errors on targets the module is unsure about
lowers \(\mathbb{E}\,p_t\) but leaves \(\mathbb{E}\,p_t^{\,n}\) to its blind spots, so a
module whose remaining errors are decisive gains accuracy without gaining
redundancy, and \((\mathbb{E}\,p_t)^n\) understates its team's failure by a factor
that grows with accuracy.

\medskip\noindent\fbox{\parbox{\dimexpr\columnwidth-2\fboxsep-2\fboxrule\relax}{\footnotesize
\textbf{Assumptions.} Theorem~\ref{thm:corr} needs positively associated errors and
a benefit that is supermodular or submodular in which beliefs are wrong (Proposition~\ref{prop:units}
gives sufficient conditions; the language models' Hanabi errors about each teammate are associated, \S\ref{sec:llm}). Theorem~\ref{thm:order} needs only that all beliefs wrong hurts
and all right helps. The gate (\S\ref{sec:gate}) is the order-2 certificate: exact, with the module's own
costs and interactions, when no interaction of three or more beliefs is non-zero. Its own
statistics flag modules whose interacting beliefs are often wrong together (mean error load,
wrong beliefs per episode, above about four, where in our data it errs toward deferral), and
there direct measurement can take over (\S\ref{sec:gate}). Like the accuracy rule it replaces, the gate is an
estimate, not a bound; it needs a simulator that accepts injected beliefs, and its tables
assume that a module's errors are unrelated to the episode state that sets their cost.}}

\section{Domains, modules and protocol}
\label{sec:setup}
\paragraph{Task assignment.} Twelve workers each hold one of ten hidden rules (global, local, hermit and
cascading). Unassigned workers self-assign to a task of
their own rule's type, which gives a strong but suboptimal passive equilibrium;
welfare decays while average satisfaction is below a target, and collapse ends the
400-round episode. Each round the coordinator assigns one worker. Harm
\(h\in[0,1]\) scales the consequence of a structurally wrong assignment, from a
wasted round (\(h=0\)) to ejecting the worker for several rounds (\(h=1\)); neither
reference policy ever makes such an assignment.

\paragraph{Level-Based Foraging (LBF)} Jumanji's implementation
\cite{papoudakis2021benchmarking,bonnet2024jumanji}, unchanged (\(10\times10\) grid,
6 agents, 8 foods, levels 1--3, 50 steps). Agents run the benchmark's own
heuristic, which is the passive equilibrium; the coordinator forms coalitions from
believed levels; harm scales LBF's failed-load penalty.

\paragraph{Hanabi.} The Hanabi Learning Environment \cite{bard2020hanabi} with four
players and the standard deck, plus a hidden hint convention per player: under
dialect A, a colour hint means ``play the newest card it touches'' and a rank hint
``keep these cards''; dialect B is the reverse. A hinter must speak the receiver's
dialect, so there are \(W=12\) directed beliefs. The belief-free team gives only
hints that are safe in either dialect; it scores 11.1 to 13.2 points out of 25, the
informed team 18.3 to 20.0. Harm is the number of life tokens: a game
that loses all its lives scores zero.

\paragraph{Belief sources.} To trace \(B\) against accuracy we corrupt each belief
to a known accuracy, independently or shared (every observer holds the same belief
about a teammate). Real modules: in task assignment, classifiers trained on 10 to
20{,}000 snapshots of one passive round (accuracy 0.42--0.53); in LBF, a module that
infers levels by watching agents; in Hanabi, a Bayesian module that infers each
dialect from how players answer others' hints, publicly, so every observer shares
its errors. Language models: Claude Haiku~4.5, Sonnet~4.5, Sonnet~5.5 and
Opus~5.5 read the public transcript of the two opening rounds of 250 fresh Hanabi
games, with both conventions described, and return a convention for every player
(four samples per model and game at temperature 1; Sonnet~5.5 and Opus~5.5 with adaptive
thinking at default effort). A teammate is \emph{hinted} if it received a hint in the
transcript, the only evidence about its convention. A \emph{clone} team gives
observer \(i\) sample \(i\) of one model, a \emph{public} team gives every observer
the same sample, and a \emph{mixed} team gives each observer a different model.
GPT-5.4-mini ran with the same prompt and its default sampling. In task assignment, Haiku
4.5 and Sonnet 5.5 label all twelve workers from the rendered state after one
passive round (300 episodes); the mixed team takes alternate workers' labels from each.

\paragraph{Protocol.} Claims marked held in Table~\ref{tab:map} were preregistered before
their tests. Accuracy sweeps use 2{,}000 episodes per cell (the one-life test, 10{,}000). A module is compared with a
surrogate that has independent errors at its own realized accuracy (Hanabi) or its
own confusion matrix (task assignment), on the same episodes. Intervals are 95\%
cluster bootstraps over episodes; harm steps use paired bootstraps.

\begin{figure*}[t]
\centering
\includegraphics[width=\textwidth]{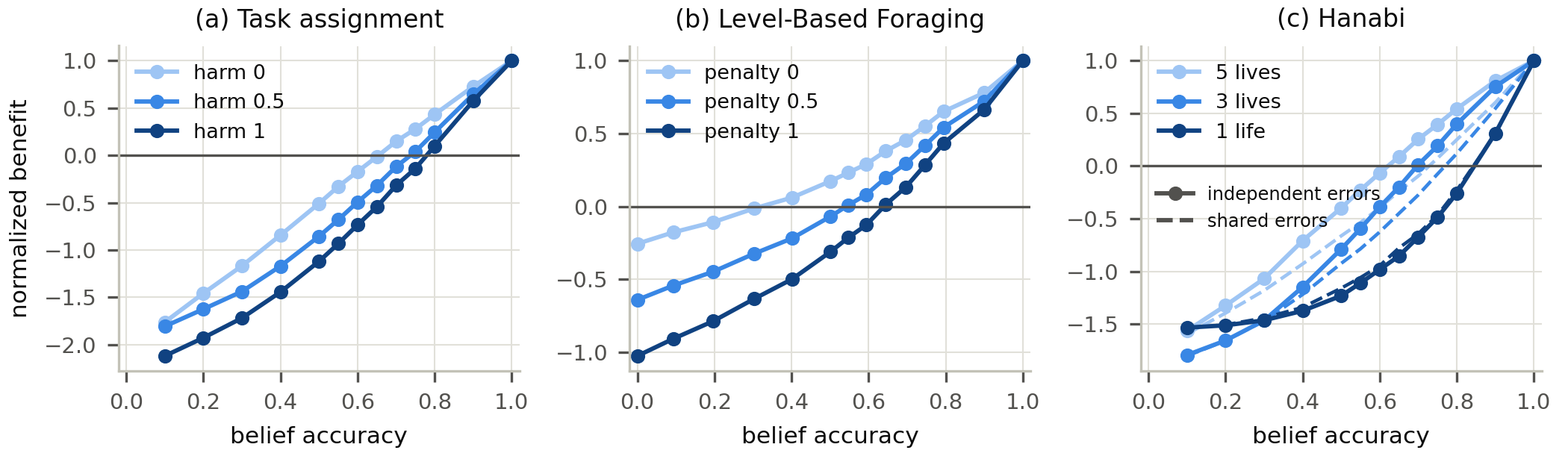}
\caption{The break-even accuracy rises with the cost of a wrong action. Normalized
benefit (0 = do nothing, 1 = informed) against belief accuracy; the threshold is where a
curve crosses zero, and even at the lowest cost a wrong action wastes a move. Thresholds: (a) 0.655 to 0.779 as harm
rises from 0 to 1; (b) 0.32 to 0.64; (c) 0.62 to 0.85 from five lives to one with
independent errors, 0.73 to 0.84 with errors shared by every observer (dashed).}
\label{fig:domains}
\end{figure*}

\section{Results}
Three results carry the paper: the threshold moves with the cost of a wrong action
(\S\ref{sec:cost}); accuracy errs in both directions, too lenient where shared errors
compound and too strict where costly errors are rare (\S\ref{sec:real}, \S\ref{sec:llm}); and accuracy buys far less
redundancy than independence predicts (\S\ref{sec:llm}, \S\ref{sec:human}). A certificate that prices the deployment rather than the module accounts for all three (\S\ref{sec:gate}).
\subsection{Cost sets the threshold}
\label{sec:cost}
Proposition~\ref{prop:mono} predicts that the break-even accuracy moves with the cost of a
wrong action and stays put for errors that cannot reach it; all three domains show the
first, and the two where such errors can be isolated show the second. In task assignment the threshold rises from 0.655 at \(h=0\) to 0.720, 0.737, 0.760
and 0.779 at \(h=0.25\) to 1 (Figure~\ref{fig:domains}a); every adjacent step is
resolved by a paired bootstrap, and Eq.~\eqref{eq:meanfield} with one free parameter
per harm explains at least 99.8\% of the variance of each curve above accuracy 0.5,
its fitted thresholds within 0.01 of the measured ones. In LBF the threshold slides from 0.32 to
0.64 as the penalty rises from 0 to 1. LBF also shows the channel half of
Proposition~\ref{prop:mono}: an under-estimated level makes the planner over-staff a
coalition, which never fails, so under-estimates never reach the penalty and their
threshold is 0.697 at every harm, while over-estimates have no knee at low harm and
become the dangerous error only above \(h\approx1.3\), as a second preregistered test
confirmed (all six out-of-range thresholds within 0.021 of their predicted values). In Hanabi the threshold rises from 0.621 with five
lives to 0.845 with one. The channel half holds in task assignment as well, under
three preregistered predictions that all held: errors that swap a worker's rule for
another of the same task type never trigger the harm, and their threshold is 0.261
at every harm, the five curves identical to the last bit, while errors that cross
types slide from 0.693 to 0.819. A knee needs a passive equilibrium a wrong action can
damage: with workers inert without a coordinator there is none, and in a \(2^4\) factorial
the Shapley value \cite{shapley1953} assigns 91\% of the knee's depth to
self-organization. The threshold can also be found
without a sweep: replaying the informed coordinator on the same episodes with a few
injected errors (ThresholdProbe) predicts it within 0.03 in every task-assignment
case, within 0.05 in 30 of 31 LBF cases and within 0.01 in Hanabi at two to five
lives.

\subsection{Real modules: accuracy errs in both directions}
\label{sec:real}
\paragraph{Task assignment: costly errors are rare.} The classifiers' accuracy (0.42--0.53)
is below the independent-error threshold at every harm, so an accuracy rule never
deploys them. Yet acting on their beliefs beats doing nothing in 33 of 35 (module,
harm) cells (32 with intervals excluding zero), by up to \(+0.76\) where independent errors at the same accuracy would
score \(-0.40\) to \(-1.00\). In the two modules we examined, errors are coherent: workers of one
class that a module gets wrong share the same wrong label in 49--54\% of pairs, against 33--46\%
for draws from its confusion matrix, and a surrogate with independent errors from the same confusion matrix
under-predicts the benefit by 0.08--0.19, mostly through which wrong labels go
together. Accuracy errs here mainly through cost: under 0.4\% of the errors violate a
task type, so few reach the costly channel, and only the deployment's costs, which the
gate's single-error tables measure (\S\ref{sec:gate}), tell which errors do. The
coherent wrong labels, outside Theorem~\ref{thm:corr}, add to the benefit.

\paragraph{Hanabi: damage compounds.} With corrupted beliefs at the same
accuracy, sharing each teammate's belief across all observers costs 0.17--0.35 of
the achievable benefit with two or more lives (16 of 16 cells, every interval
excluding zero) and raises the threshold by 0.04--0.11 (Figure~\ref{fig:domains}c).
The cost is lost redundancy: when everyone shares a misreading, no teammate can
save the target's critical cards, and 28--36\% more of them are lost. With one life
a second channel opens (shared errors misread fewer distinct teammates, so fewer
games end on a fatal misplay) and the two cancel (net \(+0.009\), CI \([-0.014,+0.030]\)).

\paragraph{Hanabi: a public module.} The Bayesian module infers in public, so every
observer shares its mistakes. It behaves like the shared-error surrogate (within 0.05 in 27 of 30 cells at
development accuracy, preregistered, and 28 at its realized accuracy), loses 0.07--0.31 to independent errors at its realized accuracy with two
or more lives (every interval excluding zero; the preregistered test, at development
accuracy, held), and accuracy approves two of its
deployments, which lose 0.23 to the belief-free team.

\begin{table*}[t]
\centering
\small
\caption{Language models as ToM modules: copies of one model fail together far more
often than independence predicts, and accuracy errs in both directions. \textbf{Hanabi}
(250 games, four samples per model): accuracy; probability that all three observers of a
teammate are wrong, and its independent prediction (error cubed); \(n_{\mathrm{eff}}\)
(Proposition~\ref{prop:neff}; 3 if independent); \(\times\)model, joint error over an
item-response model of teammate difficulty and model accuracy (\(^{\dagger}\)above its
95\% range, 0.90--1.12); benefit minus that of independent errors at the same accuracy,
five lives (\(^{*}\)CI excludes zero); accuracy-rule false approvals, one to five
lives (benefit negative, CI below zero). \textbf{Task assignment} (300 episodes, one
sample per model): coherence, look-alike workers' same-wrong-label rate over its
independent rate; benefit at \(h=0/0.5/1\), and minus that of independent errors from the
team's confusion matrix; beneficial deployments accuracy rejects (thresholds 0.655,
0.737, 0.779).}
\label{tab:llm}
\resizebox{\textwidth}{!}{%
\begin{tabular}{lccccccc}
\toprule
\multicolumn{8}{l}{\textbf{Hanabi}: shared misreadings remove redundancy (damage compounds); first value all teammates, second hinted teammates}\\
team & acc. & all three wrong & independent & \(n_{\mathrm{eff}}\) & \(\times\)model & vs.\ independent, 5 lives & false approvals \\
\midrule
GPT-5.4-mini clones & 0.631 / 0.676 & 0.178 / 0.158 & 0.050 / 0.034 & 1.73 / 1.64 & 0.92 & \(-0.15^{*}\) & 1 \\
Haiku 4.5 clones & 0.677 / 0.740 & 0.157 / 0.126 & 0.034 / 0.018 & 1.64 / 1.54 & 0.94 & \(-0.17^{*}\) & 0 \\
Sonnet 4.5 clones & 0.722 / 0.805 & 0.164 / 0.117 & 0.022 / 0.007 & 1.41 / 1.31 & \(1.19^{\dagger}\) & \(-0.19^{*}\) & 0 \\
Sonnet 5.5 clones & 0.787 / 0.875 & 0.139 / 0.090 & 0.010 / 0.002 & 1.28 / 1.16 & \(1.37^{\dagger}\) & \(-0.17^{*}\) & 1 \\
Opus 5.5 clones & 0.794 / 0.885 & 0.162 / 0.090 & 0.009 / 0.002 & 1.15 / 1.11 & \(1.66^{\dagger}\) & \(-0.19^{*}\) & 0 \\
mixed, four Claude models & 0.745 / 0.826 & 0.122 / 0.075 & 0.017 / 0.005 & 1.54 / 1.48 & 1.02 & \(-0.08\) & 0 \\
mixed, three Claude + GPT & 0.727 / 0.796 & 0.116 / 0.075 & 0.020 / 0.009 & 1.66 / 1.63 & 0.92 & \(-0.08\) & 0 \\
public (one sample for all), five models & 0.63--0.79 & & & 1 & & \(-0.20^{*}\) to \(-0.36^{*}\) & 3 \\
\midrule
\multicolumn{8}{l}{\textbf{Task assignment}: errors rarely reach the costly channel, and coherent wrong labels add to the benefit}\\
team & acc. & coherence [CI] & & & \(B\) at \(h=0/0.5/1\) & vs.\ independent, \(h=0/0.5/1\) & false rejections \\
\midrule
Haiku 4.5 & 0.330 & 1.82 [1.65, 1.99] & & & \(+0.18^{*}/+0.06/-0.13\) & \(+0.14^{*}/+0.18^{*}/+0.21^{*}\) & 1 \\
Sonnet 5.5 & 0.496 & 1.02 [0.87, 1.16] & & & \(+0.32^{*}/+0.27^{*}/+0.22^{*}\) & \(-0.06/-0.07/-0.06\) & 3 \\
mixed & 0.419 & 1.38 [1.22, 1.55] & & & \(+0.26^{*}/+0.13/-0.02\) & \(+0.07/+0.07/+0.07\) & 1 \\
\bottomrule
\end{tabular}}
\end{table*}

\begin{figure*}[t]
\centering
\includegraphics[width=\textwidth]{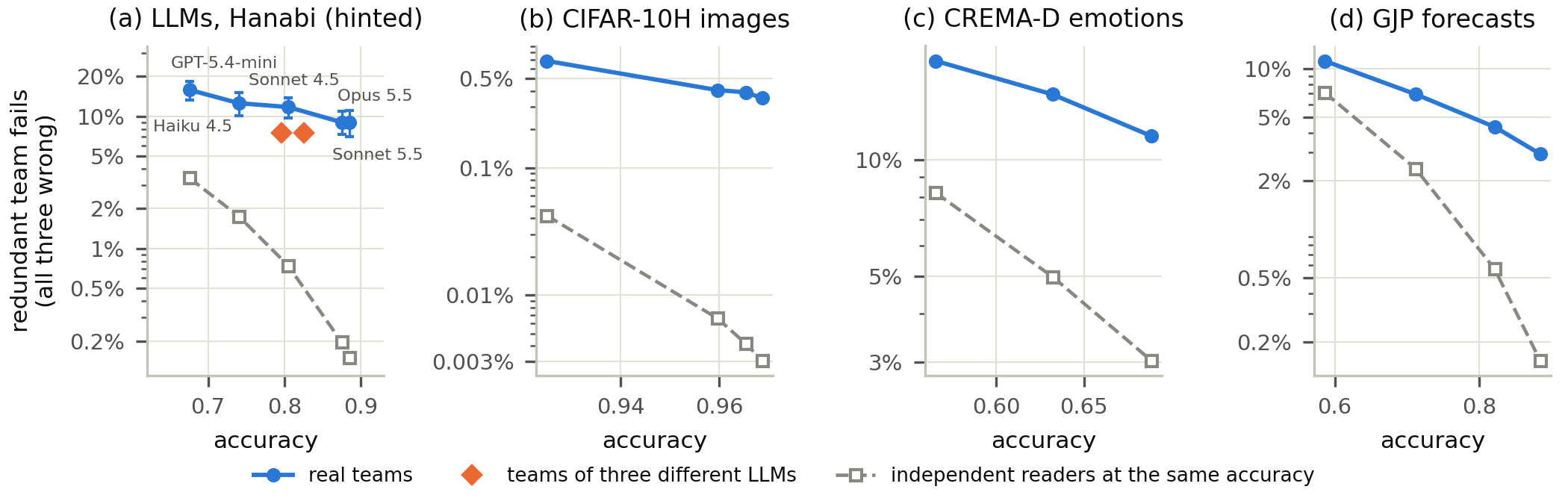}
\caption{Accurate readers still fail together: how often a redundant team of
three readers fails (all three wrong), against independent readers of the same
accuracy. (a) Language-model teams in Hanabi, teammates the models had evidence
about; bars are 95\% CIs over games. Diamonds: teams of three different models. (b--d) People:
points are reader terciles (and the top decile where enough readers rate each item),
ranked on a random half of the items and scored on the other half. Across each panel's accuracy
range, independence predicts that failure falls 2.7 to 46 times; it falls 1.6 to 3.8
times.}
\label{fig:redundancy}
\end{figure*}

\subsection{Language models: both directions, and far less redundancy than independence predicts}
\label{sec:llm}
\paragraph{Shared errors.} For every model, the three observers of a teammate are
all wrong together far more often than independence allows: 4.7 to 18.6 times for
the Claude models (preregistered, held) and 3.5 times for GPT-5.4-mini. The errors about
each teammate are associated (positive pairwise covariances; multivariate totally positive,
MTP\(_2\), joint tables), the
premise of Theorem~\ref{thm:corr}. Most of the association is shared difficulty: an
item-response model in which only a teammate's difficulty and a model's accuracy set each
error probability \cite{jo2026subjectivity,rasch1960probabilistic} already predicts joint
errors 3.8 to 11.2 times the independent rate. The rest is specific to the model and grows with accuracy (\(\times\)model in
Table~\ref{tab:llm}: no excess for Haiku and GPT-5.4-mini, 1.66 for Opus~5.5; robust to
leaving one model out and to two-parameter nulls), a within-model counterpart of the
finding that more accurate models agree more on their errors
\cite{kim2025correlated,goel2025great}. Mixing models removes this layer but not the
shared difficulty.

\paragraph{Far less redundancy than independence predicts.} On teammates the
models had evidence about, per-belief error falls from 32\% (GPT-5.4-mini) to 11\% (Opus~5.5), but the
probability that all three observers misread such a teammate falls only from 16\% to 9\%
(difference 6.8 points, CI 4.4 to 9.2), where independence predicts 3.4\% to 0.15\%:
three observers are worth 1.64 independent ones for GPT-5.4-mini and 1.11 for Opus~5.5
(Figure~\ref{fig:redundancy}a). The mechanism is Proposition~\ref{prop:neff}'s: when one
sample misreads a hinted teammate, the other three do too, 41\% of the time for
GPT-5.4-mini and 78\% for Opus~5.5. Over all teammates joint failure barely falls (0.178
and 0.162): on the quarter of teammates who received no hint, most samples answer
``A'', the convention listed first, a default every copy shares. Sharing a transcript, as deployed copies do, accounts for only part of this: four
different Claude models reading the same transcript, less accurate than Opus clones, misread a hinted teammate
jointly less often (7.5\% against 9.0\%), well under Proposition~\ref{prop:neff}(b)'s bound
of 10.3\%.

\paragraph{Consequences in both directions.} Hanabi is the compounding case. With
two or more lives every Claude clone team does worse than independent errors at its
own accuracy (12 of 12 cells, \(-0.01\) to \(-0.26\); six intervals exclude zero), and
accuracy approves five deployments whose benefit is negative with interval excluding
zero: Sonnet~5.5 clones with two lives and the Haiku and Sonnet~4.5 public teams with
five and three lives (preregistered, held) and, marginally (accuracy
inside the threshold's interval), GPT-5.4-mini's clone and public teams with five
lives. Task assignment is the cheap-error case. Every team is below
the accuracy threshold at every harm, yet Sonnet~5.5 beats not acting at all three
harms and Haiku and the mixed team do at \(h=0\): five false rejections and no false
approval (preregistered, held). Few of the models' errors violate a task type (9\% for
Haiku, 1\% for Sonnet~5.5), so they rarely reach the costly channel of
Proposition~\ref{prop:mono}, and Haiku's wrong labels are coherent: it beats independent
errors from its own confusion matrix by 0.14 to 0.21 at every harm, a gain from which
wrong labels it picks together, not from which workers it gets wrong (their pairwise
error correlation is 0.01).

\paragraph{Robustness.} Two preregistered replications probe the shared errors. With the
two conventions' definitions exchanged under their labels (5{,}000 fresh calls), the
default mostly follows the label, yet three copies are still worth only 1.08 to 1.51 observers on
hinted teammates (held), and accuracy approves four more harmful deployments. Giving half of
each team's copies the swapped labels breaks the shared default, yet those teams still
do worse than independent errors in 12 of 12 cells (three significantly). At maximum
reasoning effort (preregistered; Opus~5.5, 125 games),
Opus~5.5 spends 13 times as many tokens per answer and is no more accurate on hinted
teammates (0.869 at both efforts), its three copies are worth 1.10 observers against
1.11 (held), and only the shared default softens (``A'' in 69\% of answers on unhinted teammates against 89\%).

\subsection{Human readers}
\label{sec:human}
Human errors show the same structure (Figure~\ref{fig:redundancy}b--d).
We use CIFAR-10H \cite{peterson2019human} (2{,}571 people labelling 10{,}000 images),
CREMA-D \cite{cao2014crema} (2{,}442 people judging the emotion portrayed in 7{,}442
clips) and the Good Judgment Project \cite{mellers2014psychological} (6{,}608
forecasters, 497 resolved questions). Three readers of an item are all wrong 39, 2.7
and 2.7 times as often as independent readers of the same accuracies would be, and
with readers ranked on half the items and scored on the other half, redundant failure
falls roughly in proportion to a single reader's error as accuracy rises, not with its
cube, while series teams succeed 1.8--2.6, 11--14 and 5--9 points more often than
independence predicts. Accuracy errs in both directions here too: the most accurate tenth of
forecasters fails as a redundant trio on 2.95\% of questions where independence
predicts 0.15\%, so a failure cap between the two approves it; the top third is all
right on 64.5\% of questions where independence predicts 55.5\%, so a success
requirement between the two rejects it. The single-belief knee appears in people too: Diplomacy players judging
whether a message is a lie \cite{peskov2020takes} are right on 91.4\% of messages, below the
94.9\% of trusting every one, as near-chance lie detection \cite{bond2006accuracy} predicts; acting on suspected lies
pays only if a believed lie costs more than six false accusations.

The forecasting panel also separates a shared method from a shared information channel:
against an additive model of question difficulty and forecaster accuracy, forecasters
given the same randomly assigned training share no errors beyond those given different
training (pairs 1.00, CI \([0.996,1.006]\)), but forecasters on the same team, who see each
other's forecasts, are wrong together 1.15 times (pairs, CI 1.11--1.20) and 1.21 times
(triples, CI 1.11--1.41) as often as forecasters from different teams of the same year
(permutation \(p<0.004\); 1{,}894 forecasters in 281 teams). Team membership may not have been random, so this is an
association; it parallels the public module of
\S\ref{sec:real}: those who share an information channel share errors, those who share
only a method do not.

\begin{table}[t]
\centering
\small
\caption{The Blind-Spot Gate (E, \(\lambda=1\)) makes no significant false approval in any split,
like an A/B test with a 95\% bound but without deploying the module; accuracy (A, current
practice) makes at least one in every split. Decisions on the 57
language-model cells (48 Hanabi, 9 task assignment), estimated on a random half of each
cell's episodes and scored on the other half; means over 20 random splits (in
parentheses, s.d.\ of mean regret over splits). FA/FR: false approvals/rejections by the
held-out estimate; FA\(^{*}\)/FR\(^{*}\): where its interval excludes zero; regret:
\(\max(B,0)\) minus the value obtained; last column: splits with any FA\(^{*}\).
Rule C (dependence without interactions) equals B. The first four rows use
only injection replays and error statistics; the last two deploy the module on the
estimation half, as a paired A/B test without and with a confidence bound.}
\label{tab:gate}
\resizebox{\columnwidth}{!}{%
\begin{tabular}{lcccccccc}
\toprule
 & & & & & \multicolumn{3}{c}{mean regret} & \\
\cmidrule(lr){6-8}
rule & FA & FR & FA\(^{*}\) & FR\(^{*}\) & all & Hanabi & task & FA\(^{*}\)\,splits \\
\midrule
A: accuracy (independent threshold) & 9.5 & 7.5 & 3.2 & 5.0 & 0.056\,{\scriptsize(0.013)} & 0.033 & 0.178 & 20 \\
B: single-error costs & 15.9 & 4.2 & 8.6 & 1.7 & 0.085\,{\scriptsize(0.028)} & 0.087 & 0.076 & 20 \\
D: B + pairwise interactions & 3.8 & 5.1 & 1.5 & 1.3 & 0.023\,{\scriptsize(0.011)} & 0.016 & 0.061 & 13 \\
E: D + measured dependence (\(\lambda=1\)) & 0.7 & 9.3 & 0.0 & 2.0 & 0.021\,{\scriptsize(0.008)} & 0.014 & 0.061 & 0 \\
\midrule
direct: benefit on the estimation half & 2.1 & 2.7 & 0.3 & 0.5 & 0.008\,{\scriptsize(0.005)} & 0.007 & 0.014 & 4 \\
safety test: its 95\% lower bound \(>0\) & 0.4 & 9.1 & 0.0 & 1.8 & 0.021\,{\scriptsize(0.010)} & 0.016 & 0.048 & 0 \\
\bottomrule
\end{tabular}}
\end{table}

\begin{figure}[t]
\centering
\includegraphics[width=\columnwidth]{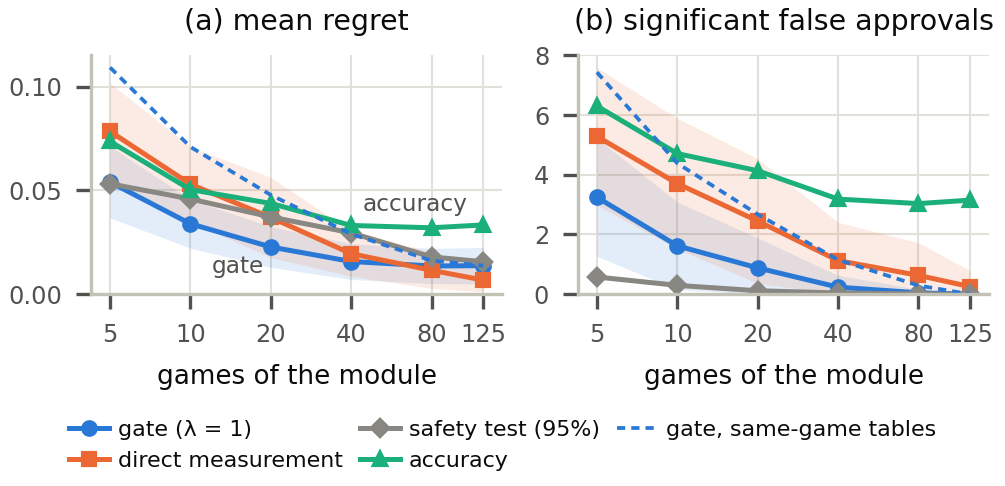}
\caption{With few games of the module the gate beats direct measurement, which catches
up at about 40 games and leads from 80 (Hanabi, 48 language-model cells).
Error statistics, the direct estimate and the safety test (Table~\ref{tab:gate}) use \(n\)
games of the estimation half; the
gate's injection tables, which need no module, the whole half (dashed: the same \(n\)
games). Means over 20 splits \(\times\) 10 draws, scored as in Table~\ref{tab:gate};
shading, \(\pm1\) s.d.\ over splits (gate, direct).}
\label{fig:fewgames}
\end{figure}

\subsection{The Blind-Spot Gate: certifying with costs, interactions and dependence}
\label{sec:gate}
A better certificate prices what accuracy throws away: truncating the expansion of
\S\ref{sec:corr} at second order gives the Blind-Spot Gate,
\[
\hat B=b_0+\sum_u q_u c_u+\sum_{u<v}\bigl(q_uq_v+\lambda\Delta_{uv}\bigr)I_{uv},
\qquad\text{act iff }\hat B>0,
\]
with \(c_u\) the cost of belief \(u\) alone wrong and \(I_{uv}\) the pairwise
interaction, both from paired single- and pair-injection replays as in
ThresholdProbe; \(q_u\) the module's error rate on belief \(u\); and
\(\Delta_{uv}=\Pr[\text{both wrong}]-q_uq_v\) its measured dependence, weighted by
\(\lambda\). At \(\lambda=1\), with the module's own costs and interactions, this is the
order-2 certificate of Theorem~\ref{thm:order}, and dependence enters only through the
interactions, which is Theorem~\ref{thm:corr} at second order. It never
deploys the module: it needs injection replays (in Hanabi independent of the
module) and the module's error statistics, so it can certify a module before deployment
wherever the coordinator can be simulated with chosen wrong beliefs.

\paragraph{As safe as an A/B test, without deploying the module.} Table~\ref{tab:gate}
scores an ablation ladder over 20 random splits of each cell's episodes, against current
practice (accuracy) and against measuring the deployment on the estimation half by a
paired A/B test. Accuracy is unsafe: it makes a significant false approval in every split,
3.2 on average. Single-error costs are worse, because a shared misreading's cost lies in the
pairs. Interactions cut significant false approvals to 1.5 and regret to 0.023 (accuracy:
0.056), and the measured dependence removes them in all 20 splits (regret 0.021), at the
price of more deferrals. That matches the safety and regret of the safety test in
Table~\ref{tab:gate}, an A/B test that lets the team act on the module's beliefs and approves
only if the one-sided 95\% lower confidence bound of high-confidence policy improvement and
quasi-Seldonian algorithms \cite{thomas2015hcpi,thomas2019preventing} is positive. Without
the bound, the A/B test has lower regret (0.008) but makes a significant false approval in
4 of 20 splits.

\paragraph{Out of sample.} Two preregistered tests (Appendix~\ref{app:oos}) apply the gate, unchanged, to
modules it was not developed on, on the same 20-split protocol. The seven task-assignment
classifiers of \S\ref{sec:real} (21 cells) are the strict case: accuracy, 0.43 to 0.54 on these episodes,
rejects every cell, while the gate approves 19 of 21 on average, among them in every split
the 18 cells of the six larger classifiers, all significantly beneficial, defers the one
harmful cell in every split, makes no significant false approval and matches direct measurement (regret
0.004 against accuracy's 0.479). The Bayesian module of \S\ref{sec:real} (24 cells, one to
five lives, one to eight rounds watched) is the lenient case: accuracy makes a significant
false approval in 17 of 20 splits, including both deployments \S\ref{sec:real} found harmful,
which the gate never approves. The gate's regret is 0.009 against 0.029, lower in 18 of 20 splits,
and its one significant false approval, in a single split, falls on a cell at the accuracy bar
where direct measurement errs the same way. The dependence term matters where the theory
predicts (exploratory): it never changes a decision on the classifiers, but without it (rule D)
the gate makes a significant false approval on the Bayesian module in 5 of 20 splits.

\paragraph{With few games of the module.} Measuring the deployment needs games of the
module in the team; the gate needs only its error statistics. Given \(n\) games in Hanabi
(Figure~\ref{fig:fewgames}), the gate has lower regret than measuring those games directly
up to 20 games (0.034 against 0.053 at 10, in at least 17 of 20 splits) and fewer
significant false approvals at every \(n\); the two are level at 40 games, and direct
measurement leads from 80. Against the safety test, the gate has lower regret from 10 to 80
games, and from 10 to 40 at a matched rate of significant false approvals (supplement).
The advantage comes from the module-free injection tables: estimated from the same \(n\)
games alone, the gate falls behind direct measurement (dashed in Figure~\ref{fig:fewgames}).

\paragraph{Scope.} The gate is second order: it drops
\(\sum_{|S|\ge3}\beta(S)\,m(S)\), which by Theorem~\ref{thm:order} can matter when three
or more interacting beliefs are often wrong together. Its own statistics, the error load,
show when a module is in that regime: handing over to direct measurement when the mean
load exceeds four errors cuts its regret from 0.021 to 0.014, at the cost of one
significant false approval in one of the 20 splits. Haiku in task assignment is the case in
point: it makes about eight errors per episode in coherent blocks, whose benefit (\(+0.18\) at \(h=0\), Table~\ref{tab:llm})
the pairwise model, accurate up to about four injected errors, cannot see, so the gate
defers it. Adding the measured within-target third-order terms (all three beliefs about
one teammate wrong) did not improve the gate's predictions at 250 games (preregistered;
Appendix~\ref{app:oos}), so the second-order gate stands.

Deferral targets what redundancy cannot fix: in Hanabi a Bayes act-or-defer rule that
withholds convention-dependent hints from teammates the module has not yet read is
beneficial in all 30 fresh-seed cells (minimum \(+0.30\); preregistered) and beats acting on
every belief in each, by at least \(+0.09\), including after a single round of watching
(\(+0.30\) against \(-0.18\) to \(-0.74\)).

\section{Discussion}
\paragraph{Certify the deployment, not the module.} What decides a deployment is what
each kind of wrong action costs there, whether the damage errors do overlaps or
compounds, and how often errors co-occur (Theorem~\ref{thm:order}); accuracy summarizes
none of them. In
practice: estimate what one wrong belief costs by paired injection, read the sign of
interactions from the task's structure (Proposition~\ref{prop:units}) or injected pairs,
measure how errors co-occur, and act on the measured deployment, with a confidence bound, when
the module can be run at length, otherwise through the gate, deferring on targets with no
evidence.

\paragraph{Redundancy is not a by-product of accuracy.} In our tasks it has to be bought
separately: diverse members or prompts raise measured redundancy (\S\ref{sec:llm}), and
shared difficulty needs more evidence or deferral (common-cause failure \cite{esary1967}
and monoculture \cite{kleinberg2021algorithmic} in ToM form). The sign of shared errors depends on the task: shared
errors that make redundant teams fragile can make teams whose errors overlap more robust
than accuracy suggests (Theorem~\ref{thm:corr}).

\paragraph{Where this applies.} Wherever several agents act on inferred beliefs and a wrong
action has a cost. Committees of language models that audit transactions, review contracts or route
cases can vote among copies of one model to catch errors; in Hanabi three copies
were worth 1.1 to 1.6 independent observers of a hinted teammate (Table~\ref{tab:llm}), and
mixing models lowered joint failure. Robot fleets and warehouse coordinators that act on
inferred intentions face the cost dial of \S\ref{sec:cost}, and there the threshold and the
gate could say when to act and when to fall back on a safe passive policy. Where a
wrong action is irreversible, as in dispatch, paired injection measures the accuracy an
agent needs before it acts rather than defers to a human. In every such setting the question
is not how accurate the module is but what its errors cost together.

\paragraph{Limitations and future work.} Our domains are controlled benchmarks with
known hidden variables, which is what lets us inject errors of known kinds; the
human data, whose hidden states nobody defined, show the same knee and the same two directions. The language-model tests use one prompt
family (relabelled once), and most human team analyses combine judgments made independently; richer prompting and teams
that deliberate are natural next tests. The gate is a second-order certificate: extending
it to modules that make many coherent errors at once is future work, and there direct
measurement applies. Across classifiers, a Bayesian module, five language models and three
sets of human judgments, whether accuracy certifies too strictly or too leniently follows
from how a task combines wrong beliefs, and can often be predicted before deployment.

\paragraph{Data availability.} The supplementary material is available as an ancillary file
with this preprint (\texttt{supplement.pdf}). It gives full methods and results, every
preregistered prediction with its outcome, the proofs and additional analyses, and its last
appendix states how AI tools were used in this research. The preregistration files, code and
result files behind each table and figure are available from the authors on request; the
third-party datasets used in the human-reader analyses are not redistributed (their sources
are listed).

\bibliographystyle{ACM-Reference-Format}
\bibliography{refs}

\appendix
\section{Proof of Theorem~\ref{thm:order}}
\label{app:thm2}
\paragraph{Setting.} There are \(P\ge2\) beliefs, each wrong value fixed as in \S\ref{sec:corr}. An error pattern is the set \(X\subseteq[P]=\{1,\dots,P\}\) of wrong beliefs, and a
module is any probability law \(\mu\) on the \(2^P\) patterns. \(\bar B(X)\) is the expected
benefit of acting given pattern \(X\), so \(\bar B(\emptyset)\) is acting with every belief
right and \(\bar B([P])\) with every belief wrong. The M\"obius coefficients
\(\beta(S)=\sum_{X\subseteq S}(-1)^{|S|-|X|}\bar B(X)\) give \(\bar B(X)=\sum_{S\subseteq X}\beta(S)\),
which is the expansion \(\bar B(\varepsilon)=\sum_S\beta(S)\prod_{p\in S}\varepsilon_p\) of
\S\ref{sec:corr}. With \(m_\mu(S)=\mu(\{X:X\supseteq S\})\), the probability that every belief in
\(S\) is wrong,
\begin{equation}
\mathbb E_\mu\bar B=\sum_X\mu(X)\sum_{S\subseteq X}\beta(S)=\sum_S\beta(S)\,m_\mu(S).
\label{eq:moebius-mean}
\end{equation}
The order-\(k\) error rates of \(\mu\) are \(r_k(\mu)=(m_\mu(S))_{1\le|S|\le k}\); order 1 gives
each belief's error rate. A rule is any map from \(r_k(\mu)\) to \{act, defer\} (it may
depend on \(\bar B\)); it decides correctly for \(\mu\) if it acts exactly when
\(\mathbb E_\mu\bar B>0\).

\paragraph{Statement.} Let \(\bar B([P])<0<\bar B(\emptyset)\) and \(k\ge1\). (i) If \(\beta(S)=0\) for every \(|S|>k\),
the rule ``act iff \(\sum_{|S|\le k}\beta(S)m(S)>0\)'' (with \(m(\emptyset)=1\)) decides correctly
for every module. (ii) If \(\beta(T)\ne0\) for some \(|T|>k\), there are \(a\in(0,1)\) and two
modules \(\mu^-,\mu^+\) with \(r_k(\mu^-)=r_k(\mu^+)=r_k(\nu_a)\), where \(\nu_a\) has independent
errors with every belief right with probability \(a\), and
\(\mathbb E_{\mu^-}\bar B<0<\mathbb E_{\mu^+}\bar B\). No rule decides correctly for both.

\begin{proof}
(i) By \eqref{eq:moebius-mean}, \(\mathbb E_\mu\bar B=\sum_{|S|\le k}\beta(S)m_\mu(S)\), a function
of \(r_k(\mu)\).

(ii) \emph{Step 1: the fibre.} For \(a\in(0,1)\) let \(F(a)=\{\mu:r_k(\mu)=r_k(\nu_a)\}\), the
laws with the order-\(k\) rates of \(\nu_a\). It is a nonempty (\(\nu_a\in F(a)\)) compact convex
polytope, and \(\mu\mapsto\mathbb E_\mu\bar B\) is linear, so the benefits of the modules in
\(F(a)\) form a closed interval \([L(a),U(a)]\) whose ends are attained.

\emph{Step 2: \(L(a)<U(a)\).} Let \(\delta_T(X)=(-1)^{|T|-|X|}\) for \(X\subseteq T\) and
\(\delta_T(X)=0\) otherwise. For any \(S\),
\[\textstyle\sum_{X\supseteq S}\delta_T(X)=\sum_{S\subseteq X\subseteq T}(-1)^{|T|-|X|}=(1-1)^{|T|-|S|},\]
which is \(1\)
if \(S=T\) and \(0\) if \(S\subsetneq T\), and the sum is empty if \(S\not\subseteq T\); so
adding \(\eta\,\delta_T\) to a law changes \(m(T)\) by \(\eta\) and no other \(m(S)\), including
the total mass \(m(\emptyset)\). Every pattern has positive probability under \(\nu_a\)
(\(\nu_a(X)=(1-a)^{|X|}a^{P-|X|}\)), so for \(|\eta|\le\eta_0=\min_X\nu_a(X)\) the law
\(\nu_a+\eta\delta_T\) is a probability law, and it lies in \(F(a)\) because \(|T|>k\). By
\eqref{eq:moebius-mean} its benefit is \(\mathbb E_{\nu_a}\bar B+\eta\beta(T)\); taking
\(\eta=\pm\eta_0\) gives two benefits \(2\eta_0|\beta(T)|>0\) apart, so \(L(a)<U(a)\).

\emph{Step 3: continuity.} Let \(D\) be the set of rate vectors \(r_k(\mu)\) of all laws
\(\mu\), a convex polytope (the image of the simplex under a linear map), and
\(V(r)=\max\{\mathbb E_\mu\bar B:r_k(\mu)=r\}\) for \(r\in D\). \(V\) is finite and concave on
\(D\): if \(\mu_1,\mu_2\) attain \(V(r_1),V(r_2)\), the mixture
\(\theta\mu_1+(1-\theta)\mu_2\) has rates \(\theta r_1+(1-\theta)r_2\) and benefit
\(\theta V(r_1)+(1-\theta)V(r_2)\), so \(V(\theta r_1+(1-\theta)r_2)\ge\theta V(r_1)+(1-\theta)V(r_2)\).
A finite concave function is continuous on the relative
interior of its domain \cite[Theorem~10.1]{rockafellar1970convex}, and
\(r_k(\nu_a)\) lies in the relative interior of \(D\) because \(\nu_a\) lies in the relative
interior of the simplex and a linear map sends relative interiors onto relative interiors
\cite[Theorem~6.6]{rockafellar1970convex}. As \(a\mapsto r_k(\nu_a)\) is continuous (its entries are
\((1-a)^{|S|}\)), \(U(a)=V(r_k(\nu_a))\) is continuous on \((0,1)\), and so is \(L(a)\), the same
argument applied to \(-\bar B\).

\emph{Step 4: the ends.} Every \(\mu\in F(a)\) has each belief wrong with probability
\(1-a\) (order-1 rates, \(k\ge1\)), so \(\mu(X\ne[P])\le\sum_p\Pr_\mu[p\text{ right}]=Pa\) and
\(\mathbb E_\mu\bar B\le\bar B([P])+Pa\,(\max\bar B-\bar B([P]))\). Hence \(U(a)<0\) for
\(0<a<a_0=-\bar B([P])/\bigl(P(\max\bar B-\bar B([P]))\bigr)\). Likewise
\(\mu(X\ne\emptyset)\le P(1-a)\) and \(\mathbb E_\mu\bar B\ge\bar B(\emptyset)-P(1-a)(\bar
B(\emptyset)-\min\bar B)\), so \(L(a)>0\) for \(a>1-a_1\),
\(a_1=\bar B(\emptyset)/\bigl(P(\bar B(\emptyset)-\min\bar B)\bigr)\).

\emph{Step 5: a sign change inside one fibre.} Let \(a^\ast=\sup\{a\in(0,1):U(a)\le0\}\). The
set contains \((0,a_0)\) and no point of \((1-a_1,1)\), where \(U\ge L>0\), so
\(a^\ast\in(0,1)\); by continuity \(U(a^\ast)\le0\), and \(U(a)>0\) for every \(a>a^\ast\). Since
\(L(a^\ast)<U(a^\ast)\le0\) and \(L\) is continuous, \(L(a)<0\) on some \([a^\ast,a^\ast+\epsilon)\);
for \(a\) in \((a^\ast,a^\ast+\epsilon)\), \(L(a)<0<U(a)\). Take \(\mu^-\) and \(\mu^+\) attaining
\(L(a)\) and \(U(a)\). A rule sees the same rates for both and so gives them the same
decision, which is wrong for one of them.
\end{proof}

\paragraph{Remarks.}
Step 2 uses the fact that \(k\)-wise marginals do not fix the mean of terms of degree
above \(k\) \cite{odonnell2014analysis}; Steps 3 to 5 turn it into a
sign change between modules with the same order-\(k\) rates, here even the same accuracy on
every belief, whenever acting can hurt. The theorem guarantees such a pair exists, not that the range of accuracies where
it exists is wide; Figure~\ref{fig:domains}c shows a wide one in Hanabi. For \(k=1\): accuracy (each belief's error rate, even if every belief has the same one) can
certify every module only if \(\beta(S)=0\) for all \(|S|\ge2\), that is, only if no wrong
beliefs interact. The gate of \S\ref{sec:gate} at \(\lambda=1\), with the module's own costs and
interactions, computes \(\sum_{|S|\le2}\beta(S)m(S)\) from injection replays and the module's
error statistics, so it is the \(k=2\) case: exact when interactions stop at pairs. Figure~\ref{fig:domains}c exhibits a pair of modules of the kind in (ii): with five lives, at
accuracies between the independent-error threshold (0.621, interval 0.610 to 0.633) and the
shared-error threshold (0.728, interval 0.716 to 0.741), independent errors benefit and
errors shared by every observer harm (at accuracy 0.65: \(+0.09\) against \(-0.24\); at 0.70:
\(+0.26\) against \(-0.09\)). The gate's exactness also needs its tables to be the module's
own, that is, errors whose effect does not depend on the episode they occur in and wrong
values distributed as in the replays (coherent wrong labels in task assignment are not).
As a numerical check, \texttt{theory/thm2\_check.py} (supplement) computes \(L(a)\) and \(U(a)\) by linear
programming on a grid of accuracies for 120 random tasks with 2 to 5 beliefs and \(k\) from 1
to \(P\); in every task the two ends coincide at every accuracy when the M\"obius degree is at
most \(k\), and otherwise an accuracy with \(L(a)<0<U(a)\) exists.

\section{Out-of-sample certification}
\label{app:oos}
Both tests apply the gate of \S\ref{sec:gate} unchanged (\(\lambda=1\), class-independent injection
tables, no free parameter) to modules it was not developed on, with the predictions written
before any replay; rules, splits and scoring are those of Table~\ref{tab:gate}.

\paragraph{Trained classifiers.} The seven task-assignment classifiers of \S\ref{sec:real}, retrained
with the published recipe (their held-out accuracies reproduce the archived values), label the 300
episodes of the gate table; injection tables come from single- and pair-injection replays with wrong
labels drawn from each classifier's own confusion row (21 cells). Both predictions held: no
significant false approval in any of 20 splits (OOS1), and regret below accuracy's in all 20 (OOS2;
at least 15 predicted).

\begin{table}[h]
\centering\small
\caption{Classifiers, 21 cells, means over 20 random splits (in parentheses, s.d.\ of mean regret over
splits). Accuracy (0.43 to 0.54) is below every threshold and rejects every cell.
\(^{\dagger}\)Exploratory, computed after the preregistered test was scored; D and E decide
identically on every cell.}
\label{tab:arxiv-oos}
\resizebox{\columnwidth}{!}{%
\begin{tabular}{lccccccc}
\toprule
rule & FA & FR & FA\(^{*}\) & FR\(^{*}\) & regret & approved & FA\(^{*}\) splits \\
\midrule
A: accuracy (independent threshold) & 0.0 & 19.7 & 0.0 & 17.8 & 0.479\,{\scriptsize(0.038)} & 0.0 & 0 \\
B: single-error costs\(^{\dagger}\) & 0.0 & 2.7 & 0.0 & 1.0 & 0.019\,{\scriptsize(0.016)} & 17.0 & 0 \\
D: B + pairwise interactions\(^{\dagger}\) & 0.3 & 0.8 & 0.0 & 0.2 & 0.004\,{\scriptsize(0.005)} & 19.2 & 0 \\
E: the gate (\(\lambda=1\)) & 0.3 & 0.8 & 0.0 & 0.2 & 0.004\,{\scriptsize(0.005)} & 19.2 & 0 \\
direct: benefit on the estimation half & 0.4 & 0.7 & 0.0 & 0.2 & 0.004\,{\scriptsize(0.008)} & 19.4 & 0 \\
\bottomrule
\end{tabular}}
\end{table}

\paragraph{Bayesian module.} The public module of \S\ref{sec:real}, run on the 250 gate games at one,
two, three and five lives after one to eight rounds of watching (24 cells), with injection tables
replayed at the same watching time. Accuracy made at least one significant false approval in 17 of 20
splits (OOSH2; at least 10 predicted) and the gate's regret was lower in 18 of 20 (OOSH3; at least 15
predicted). The gate made one significant false approval, in one split, at one life after three rounds,
where the module's accuracy (0.851) sits at the bar (0.845) and direct measurement made the same one;
the prediction of none in any split (OOSH1) was therefore not met. Without the dependence
term (rule D) the gate made a significant false approval in 5 of the 20 splits, and with
single-error costs alone (rule B) in all 20; on the classifiers, whose errors are nearly
independent, D and E decide identically.

\begin{table}[h]
\centering\small
\caption{Bayesian module, 24 cells, means over 20 random splits (in parentheses, s.d.\ of mean
regret over splits; exact halves rounded up). E\(_K\) hands over to direct
measurement at a mean error load above four; the last row adds the measured within-target
third-order terms. \(^{\dagger}\)Exploratory, computed after the preregistered test was scored.}
\label{tab:arxiv-oosh}
\resizebox{\columnwidth}{!}{%
\begin{tabular}{lccccccc}
\toprule
rule & FA & FR & FA\(^{*}\) & FR\(^{*}\) & regret & approved & FA\(^{*}\) splits \\
\midrule
A: accuracy (fixed bar per life setting) & 4.7 & 0.2 & 1.10 & 0.00 & 0.029\,{\scriptsize(0.013)} & 19.5 & 17 \\
B: single-error costs\(^{\dagger}\) & 4.6 & 0.6 & 2.10 & 0.05 & 0.053\,{\scriptsize(0.020)} & 19.0 & 20 \\
D: B + pairwise interactions\(^{\dagger}\) & 2.1 & 0.6 & 0.25 & 0.00 & 0.010\,{\scriptsize(0.010)} & 16.6 & 5 \\
E: the gate (\(\lambda=1\)) & 1.1 & 1.4 & 0.05 & 0.20 & 0.009\,{\scriptsize(0.008)} & 14.7 & 1 \\
E\(_K\): direct above load 4 & 1.1 & 1.4 & 0.05 & 0.20 & 0.009\,{\scriptsize(0.008)} & 14.7 & 1 \\
direct: estimation-half benefit & 2.0 & 1.3 & 0.05 & 0.05 & 0.009\,{\scriptsize(0.006)} & 15.7 & 1 \\
E + within-target third order & 1.0 & 1.4 & 0.05 & 0.20 & 0.009\,{\scriptsize(0.010)} & 14.6 & 1 \\
\bottomrule
\end{tabular}}
\end{table}

\paragraph{A third-order correction.} Theorem~\ref{thm:order} says the order-2 gate is exact only when
no interaction of three or more beliefs is non-zero. In Hanabi the natural candidate is the three
beliefs about one teammate, whose interaction \(\beta_3(j)\) we measured with block injections (all
three wrong) and added as \(\sum_j\beta_3(j)\,m_3(j)\). The preregistered prediction that this would
lower the prediction error on the 48 language-model cells in at least 15 of 20 splits failed (6 of 20;
mean absolute error 0.308 against 0.292): at 250 games the measured \(\beta_3\) are small and noisy
(\(|\beta_3|\approx0.05\)), so the correction adds estimation noise and the second-order gate stands.

\end{document}